\pdfoutput=1
\documentclass{article}
\usepackage[final,nonatbib]{nips_2017}
\usepackage[numbers]{natbib}
\usepackage[utf8]{inputenc} 
\usepackage[T1]{fontenc}    
\usepackage{hyperref}       
\usepackage{url}            
\usepackage{booktabs}       
\usepackage{amsfonts}       
\usepackage{nicefrac}       
\usepackage{microtype}      

\usepackage{cleveref}

\usepackage{graphicx}
\usepackage{caption}
\usepackage{subcaption}
\usepackage{algorithm}
\usepackage{algorithmic}
\usepackage{hyperref}
\usepackage{cleveref}
\usepackage{placeins}
\usepackage{dsfont}
\usepackage{amsthm}
\usepackage{amsmath}
\usepackage{amssymb}
\usepackage{enumitem}
\usepackage{tabularx}
\usepackage{todonotes}

\newcommand\blfootnote[1]{%
  \begingroup
  \renewcommand\thefootnote{}\footnote{#1}%
  \addtocounter{footnote}{-1}%
  \endgroup
}

\newcommand{\fig}[1]{Figure~\ref{fig:#1}}

\newcommand{\eq}[1]{(\ref{eq:#1})}

\newcommand{\x}{{\ensuremath{{\mathbf x}}}} 
\newcommand{\z}{{\ensuremath{{\mathbf z}}}} 
\newcommand{\w}{{\ensuremath{{\mathbf w}}}} 
\newcommand{\q}{{\ensuremath{{\mathbf q}}}}

\newcommand{\sX}{{\ensuremath{\cal X}}} 
\newcommand{\sZ}{{\ensuremath{\cal Z}}} 
\newcommand{\dX}{{\ensuremath{X}}} 
\newcommand{\e}{{\ensuremath{e}}} 
\newcommand{\g}{{\ensuremath{g}}} 
\newcommand{\pxz}{{\ensuremath{\psi}}} 
\newcommand{\pzx}{{\ensuremath{\theta}}} 
\newcommand{\loss}{{L}} 

\renewcommand{\P}[1]{{\ensuremath{{#1}}}}

\newcommand{\Pegz}{\P{e(g(Z))}} 
\newcommand{\Pgz}{\P{g(Z)}}
\newcommand{\Pex}{\P{e(X)}}
\newcommand{\Pz}{\P{Z}}
\newcommand{\Px}{\P{X}}

\newcommand{\eofx}[1]{{\e_{\pxz}(#1)}}
\newcommand{\gofz}[1]{{\g_{\pzx}(#1)}}

\newcommand{\sph}{\ensuremath{{{\mathds S}^M}}} 
\newcommand{\E}{{\mathbb{E}}} 
\newcommand{\KL}{{{\text{KL}}}} 

\newtheorem{theorem}{Theorem}

\title{It Takes (Only) Two:\\
Adversarial Generator-Encoder Networks}

\author{
  Dmitry Ulyanov\\
  Skolkovo Institute of Science and Technology, Yandex\\
  \texttt{dmitry.ulyanov@skoltech.ru} \\
  \And
  Andrea Vedaldi \\
  University of Oxford \\
  \texttt{vedaldi@robots.ox.ac.uk} \\
  \And
  Victor Lempitsky  \\
  Skolkovo Institute of Science and Technology \\
  \texttt{lempitsky@skoltech.ru} \\
}

\begin{document} 
\maketitle
\begin{abstract} 
We present a new autoencoder-type architecture that is trainable in an unsupervised mode, sustains both generation and inference, and has the quality of conditional and unconditional samples boosted by adversarial learning. Unlike previous hybrids of autoencoders and adversarial networks, the adversarial game in our approach is set up directly between the encoder and the generator, and no external mappings are trained in the process of learning.
The game objective compares the divergences of each of the real and the generated data distributions with the prior distribution in the latent space. We show that direct generator-vs-encoder game leads to a tight coupling of the two components, resulting in samples and reconstructions of a comparable quality to some recently-proposed more complex architectures. \blfootnote{The source code is available at \url{https://github.com/DmitryUlyanov/AGE}}
\end{abstract} 

\section{Introduction}

Deep (Variational) Auto Encoders (AEs~\cite{Bengio09} and VAEs~\cite{Kingma13,Rezende14}) and deep Generative Adversarial Networks (GANs~\cite{Goodfellow14}) are two of the most popular approaches to generative learning. These methods have complementary strengths and weaknesses. VAEs can learn a \textit{bidirectional} mapping between a complex data distribution and a much simpler prior distribution, allowing both generation and inference; on the contrary, the original formulation of GAN learns a \emph{unidirectional} mapping that only allows sampling the data distribution. On the other hand, GANs use more complex loss functions compared to the simplistic data-fitting losses in (V)AEs and can usually generate more realistic samples.

Several recent works have looked for hybrid approaches to support, in a principled way, both sampling and inference like AEs, while producing samples of quality comparable to GANs. Typically this is achieved by training a AE jointly with one or more adversarial discriminators whose purpose is to improve the alignment of distributions in the latent space~\cite{Brock16,Makhzani15}, the data space~\cite{Che16,Larsen15} or in the joint (product) latent-data space~\cite{Donahue16,Dumoulin16}. Alternatively, the method of~\cite{Zhu16} starts by learning a unidirectional GAN, and then learns a corresponding inverse mapping (the encoder) post-hoc.

While compounding autoencoding and adversarial discrimination does improve GANs and VAEs, it does so at the cost of added complexity. In particular, each of these systems involves at least three deep mappings: an encoder, a decoder/generator, and a discriminator. In this work, we show that this is unnecessary and that the advantages of autoencoders and adversarial training can be combined without increasing the complexity of the model.

In order to do so, we propose a new architecture, called an \emph{Adversarial Generator-Encoder (AGE) Network} (\cref{s:method}), that contains only two feed-forward mappings, the encoder and the generator, operating in opposite directions. As in VAEs, the generator maps a simple prior distribution in latent space to the data space, while the encoder is used to move both the real and generated data samples into the latent space. In this manner, the encoder induces two latent distributions, corresponding respectively to the \textit{encoded real data} and the \textit{encoded generated data}. The AGE learning process then considers the divergence of each of these two distributions to the original prior distribution.

There are two advantages of this approach. First, due to the simplicity of the prior distribution, computing its divergence to the latent data distributions reduces to the calculation of simple statistics over small batches of images. Second, unlike GAN-like approaches, real and generated distributions are never compared directly, thus bypassing the need for discriminator networks as used by GANs. Instead, the adversarial signal in AGE comes from learning the encoder to increase the divergence between the latent distribution of the generated data and the prior, which works against the generator, which tries to decrease the same divergence (\fig{scheme}). Optionally, AGE training may include reconstruction losses typical of AEs.

The AGE approach is evaluated (\cref{s:exp}) on a number of standard image datasets, where we show that the quality of generated samples is comparable to that of GANs~\cite{Goodfellow14,Radford15}, and the quality of reconstructions is comparable or better to that of the more complex Adversarially-Learned Inference (ALI) approach of~\cite{Dumoulin16}, while training faster. We further evaluate the AGE approach in the conditional setting, where we show that it can successfully tackle the colorization problem that is known to be difficult for GAN-based approaches. Our findings are summarized in~\cref{s:conc}.

\textbf{Other related work.} Apart from the above-mentioned approaches, AGE networks can be related to several other recent GAN-based systems. Thus, they are related to improved GANs~\cite{Salimans16} that proposed to use batch-level information in order to prevent mode collapse. The divergences within AGE training are also computed as batch-level statistics. 

Another avenue for improving the stability of GANs has been the replacement of the classifying discriminator with the regression-based one as in energy-based GANs~\cite{Zhao16} and Wasserstein GANs~\cite{Arjovsky17}. Our statistics (the divergence from the prior distribution) can be seen as a very special form of regression. In this way, the encoder in the AGE architecture can be (with some reservations) seen as a discriminator computing a single number similarly to how it is done in \cite{Arjovsky17,Zhao16}.
\section{Adversarial Generator-Encoder Networks}\label{s:method}

\begin{figure}
    \centering
    \includegraphics[width=0.7\columnwidth]{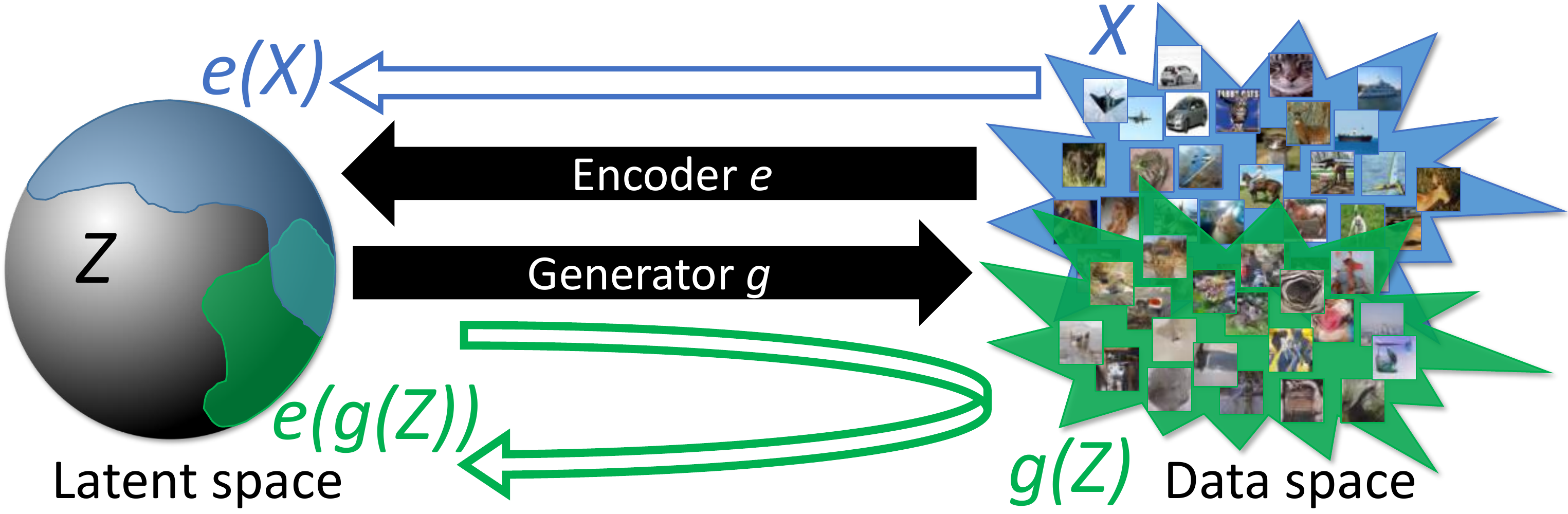}
    \caption{Our model (AGE network) has only two components: the generator $\g$ and the encoder $\e$. The learning process adjusts their parameters in order to align a simple prior distribution $\Pz$ in the latent space and the data distribution $\Px$. This is done by adversarial training, as updates for the generator aim to minimize the divergence between $\Pegz$ and $\Pz$ (aligning green with gray), while updates for the encoder aim to minimize the divergence between $\Pex$ (aligning blue with gray) and to \textit{maximize} the divergence between $\Pegz$ and $\Pz$ (shrink green ``away'' from gray). We demonstrate that such adversarial learning gives rise to high-quality generators that result in the close match between the real distribution $\Px$ and the generated distribution $\Pgz$. Our learning can also incorporate reconstruction losses to ensure that encoder-generator acts as autoencoder (section \ref{sec:reconstruction}). }
    \label{fig:scheme}
\end{figure}

This section introduces our Adversarial Generator-Encoder (AGE) networks. An AGE is composed of two parametric mappings: the \textit{encoder} $\eofx{\x}$, with the learnable parameters $\pxz$, that maps the data space \sX{} to the latent space \sZ{}, and the \textit{generator} $\gofz{\z}$, with the learnable parameters \pzx{}, which runs in the opposite direction. We will use the shorthand notation $f(Y)$ to denote the distribution of the random variable $f(\mathbf{y}),$ $\mathbf{y}\sim Y$.

The reference distribution $\Pz$ is chosen so that it is easy to sample from it, which in turns allow to sample $\gofz{Z}$ unconditionally be first sampling $\z \sim \Pz$ and then by feed-forward evaluation of $\x = \gofz{\z}$, exactly as it is done in GANs. In our experiments, we pick the latent space \sZ{} to be an $M$-dimensional sphere \sph, and the latent distribution to be a uniform distribution on that sphere $\Pz = {\text{Uniform}}(\sph)$. We have also conducted some experiments with the unit Gaussian distribution in the Euclidean space and have obtained results comparable in quality.

The goal of learning an AGE is to align the real data distribution $\Px$ to the generated distribution $\gofz{Z}$ while establishing a correspondence between data and latent samples $\x$ and $\z$. The real data distribution $\Px$ is empirical and represented by a large number $N$ of data samples $\{\x_1,\x_2,...\x_N\}$. Learning amounts to tuning the parameter \pxz{} and \pzx{} to optimize the AGE criterion, discussed in~\cref{sec:alignment}. This criterion is based on an adversarial game whose saddle points correspond to networks that align real and generated data distribution ($\Pgz = \Px$). The criterion is augmented with additional terms that encourage the reciprocity of the encoder $\e$ and the generator $\g$~(\cref{sec:reconstruction}). The details of the training procedure are given in~\cref{sec:training}.

\subsection{Adversarial distribution alignment}\label{sec:alignment}

The GAN approach to aligning two distributions is to define an adversarial game based on a ratio of probabilities~\cite{Goodfellow14}. The ratio is estimated by repeatedly fitting a binary classifier that distinguishes between samples obtained from the real and generated data distributions. Here, we propose an alternative adversarial setup with some advantages with respect to GAN's, including avoiding generator collapse~\cite{Goodfellow17}.

The goal of AGE is to generate a distribution $\Pgz$ in data space that is close to the true data distribution $\Px$. However, direct matching of the distributions in the high-dimensional data space, as done in GAN, can be challenging. We propose instead to move this comparison \emph{to the simpler latent space}. This is done by introducing a divergence measure $\Delta(P \| Q)$ between distributions defined in the latent space $\sZ$. We only require this divergence to be non-negative and zero if, and only if, the distributions are identical ($\Delta(P \| Q) = 0  \Longleftrightarrow P=Q$).\footnote{We do not require the divergence to be a distance.} The encoder function $\e$ maps the distributions $\Px$ and $\Pgz$ defined in data space to corresponding distributions $\Pex$ and $\Pegz$ in the latent space. Below, we show how to design an adversarial criterion such that minimizing the divergence $\Delta(\Pex,\Pegz)$ in latent space induces the distributions $\Px$ and $\Pgz$ to align in data space as well.

In the theoretical analysis below, we assume that encoders and decoders span the class of all measurable mappings between the corresponding spaces. This assumption, often referred to as \textit{non-parametric limit}, is justified by the universality of neural networks \cite{Hornik1989359}. We further make the \textbf{assumption} that there exists at least one ``perfect'' generator that matches the data distribution, i.e.\ $\exists \g_0: \g_0(Z)=X$. 

We start by considering a simple game with objective defined as:  
\begin{equation}\label{eq:game1}
\max_e \min_g V_1(g,e) = \Delta(\,\Pegz \| \Pex\,)\, .
\end{equation}
As the following theorem shows, perfect generators form saddle points (Nash equilibria) of the game~\eq{game1} and all saddle points of the game \eq{game1} are based on perfect generators.

\begin{theorem}\label{th:game1}
A pair $(\g^*,\e^*)$ forms a saddle point of the game \eq{game1} if and only if the generator $\g^*$ matches the data distribution, i.e.\ $\g^*(Z)=X$. 
\end{theorem}
The proofs of this and the following theorems are given in the supplementary material.

While the game~\eq{game1} is sufficient for aligning distributions in the data space, finding such saddle points is difficult due to the need of comparing two empirical (hence non-parametric) distributions $\Pex$ and $\Pegz$. We can avoid this issue by introducing an intermediate reference distribution $Y$ and comparing the distributions to that instead, resulting in the game:
\begin{equation}\label{eq:game2}
\max_e \min_g V_2(\g,\e) = \Delta(\Pegz  \| Y) - \Delta(\Pex \| Y).  
\end{equation}
Importantly,~\eq{game2} still induces alignment of real and generated distributions in data space:
\begin{theorem}\label{th:game2}
If a pair $(\g^*,\e^*)$ is a saddle point of game \eq{game2} then the generator $\g^*$ matches the data distribution, i.e.\ $g^*(Z) = X$. Conversely, if the generator $\g^*$ matches the data distribution, then for \textit{some} $\e^*$ the pair $(\g^*,\e^*)$ is a saddle point of \eq{game2}.  
\end{theorem}

The important benefit of formulation~\eq{game2} is that, if $Y$ is selected in a suitable manner, it is simple to compute the divergence of $Y$ to the empirical distributions $e(g(Z))$ and $e(X)$. For convenience, in particular, we choose $Y$ to coincide with the ``canonical'' (prior) distribution $Z$. By substituting $Y=Z$ in objective~\eq{game2}, the loss can be extended to include reconstruction terms that can improve the quality of the result. It can also be optimized by using stochastic approximations as described in section \ref{sec:training}. 

Given a distribution $Q$ in data space, the encoder $e$ and divergence $\Delta(\cdot\|Y)$ can be interpreted as extracting statistics $F(Q) = \Delta(\e(Q)\|Y)$ from $Q$. Hence, game~\eq{game2} can be though of as comparing certain statistics of the real and generated data distributions. Similarly to GANs, these statistics are not fixed but evolve during learning.

We also note that, even away from the saddle point, the minimization $\min_\g V_2(\g, \e)$ for a fixed $\e$ does not tend to collapse for many reasonable choice of divergence (e.g.\ KL-divergence). In fact, any collapsed distribution would inevitably lead to a very high value of the first term in~\eq{game2}. Thus, unlike GANs, our approach can optimize the generator for a fixed adversary till convergence and obtain a non-degenerate solution. On the other hand, the maximization $\max_e V_2(\g,\e)$ for some fixed $\g$ can lead to $+\infty$ score for some divergences.

\subsection{Encoder-generator reciprocity and reconstruction losses}\label{sec:reconstruction}

In the previous section we have demonstrated that finding a saddle point of~\eq{game2} is sufficient to align real and generated data distributions $\Px$ and $\Pgz$ and thus generate realistically-looking data samples. At the same time, this by itself does not necessarily imply that mappings $\e$ and $\g$ are reciprocal. Reciprocity, however, can be desirable if one wishes to reconstruct samples  $\x=g(\z)$ from their codes $\z=e(\x)$.

In this section, we introduce losses that encourage encoder and generator to be reciprocal. Reciprocity can be measured either in the latent space or in the data space, resulting in the loss functions based on reconstruction errors, e.g.:
\begin{eqnarray}
\loss_\sX(\g_\pzx,\e_\pxz) = \E_{\x \sim \Px} \| \x - \g_\pzx \left(\strut\e_\pxz(\x)\right) \|_1 \label{eq:dataloss}\,, \\
\loss_\sZ(\g_\pzx,\e_\pxz) = \E_{\z \sim \Pz} \| \z - \e_\pxz\left(\strut\g_\pzx(\z)\right) \|^2_2 \,. \label{eq:latentloss}
\end{eqnarray}
Both losses \eq{dataloss} and \eq{latentloss} thus encourage the reciprocity of the two mappings. Note also that \eq{dataloss} is the traditional pixelwise loss used within AEs (L1-loss was preferred, as it is known to perform better in image synthesis tasks with deep architectures).

A natural question then is whether it is helpful to minimize both losses \eq{dataloss} and \eq{latentloss} at the same time or whether considering only one is sufficient. The answer is given by the following statement:

\begin{theorem}\label{th:th3}
Let the two distributions $W$ and $Q$ be aligned by the mapping $f$ (i.e.\ $f(W) = Q$) and let $\E_{\w \sim W} \| \w - h \left(\strut f(\w)\right) \|_2^2 = 0$. Then, for $\w \sim W$ and $\q \sim Q$, we have $\w = h(f(\w))$ and $\q = f(h(\q))$ almost certainly, i.e.\ the mappings $f$ and $h$ invert each other almost everywhere on the supports of $W$ and $Q$. Furthermore, $Q$ is aligned with $W$ by $h$, i.e.\ $h(Q) = W$.
\end{theorem}

Recall that Theorem~\ref{th:game2} establishes that the solution (saddle point) of game~\eq{game2} aligns distributions in the data space. Then Theorem~\ref{th:th3} shows that when augmented with the latent space loss~\eq{latentloss},  the objective \eq{game2} is sufficient to ensure reciprocity.

\subsection{Training AGE networks}\label{sec:training} 

Based on the theoretical analysis derived in the previous subsections, we now suggest the approach to the joint training of the generator in the encoder within the AGE networks. As in the case of GAN training, we set up the learning process for an AGE network as a game with the  iterative updates over the parameters $\pzx$ and $\pxz$ that are driven by the optimization of different objectives. In general, the optimization process combines the maximin game for the functional \eq{game2} with the optimization of the reciprocity losses \eq{dataloss} and \eq{latentloss}. 

In particular, we use the following game objectives for the generator and the encoder:
\begin{eqnarray}
\hat{\pzx} = \arg \min_\pzx \left( V_2(\g_\pzx,\e_{\bar{\pxz}}) + \lambda \loss_\sZ(\g_\pzx,\e_{\bar{\pxz}})  \right) \,, \label{eq:pzxopt} \\
\hat{\pxz} = \arg \max_\pxz \left( V_2(\g_{\bar{\pzx}},\e_\pxz) - \mu \loss_\sX(\g_{\bar{\pzx}},\e_\pxz)  \right) \,, \label{eq:pxzopt}
\end{eqnarray}
where $\bar{\pxz}$ and $\bar{\pzx}$ denote the value of the encoder and generator parameters at the moment of the optimization and $\lambda$, $\mu$ is a user-defined parameter. Note that both objectives \eq{pzxopt}, \eq{pxzopt} include only one of the reconstruction losses. Specifically, the generator objective includes only the latent space reconstruction loss. In the experiments, we found that the omission of the other reconstruction loss (in the data space) is important to avoid possible blurring of the generator outputs that is characteristic to autoencoders. Similarly to GANs, in \eq{pzxopt}, \eq{pxzopt} we perform only several steps toward optimum at each iteration, thus alternating between generator and encoder updates.

By maximizing the difference between $\Delta(\e_\pxz(\g_{\bar {\pzx}}(\Pz))\|\Pz)$ and $\Delta(\e_\pxz(\dX)\|\Pz)$, the optimization process \eq{pxzopt} focuses on the maximization of the mismatch between the real data distribution $\Px$ and the distribution of the samples from the generator $\g_{\bar {\pzx}}(\Pz)$. Informally speaking, the optimization \eq{pxzopt} forces the encoder to find the mapping that aligns real data distribution $\Px$ with the target distribution $\Pz$, while mapping non-real (synthesized data) $\g_{\bar {\pzx}}(\Pz)$ away from $\Pz$. When $\Pz$ is a uniform distribution on a sphere, the goal of the encoder would be to uniformly spread the real data over the sphere, while cramping as much of synthesized data as possible together assuring non-uniformity of the distribution $\e_\pxz\left(\strut\g_{\bar{\pzx}}(Z)\right)$.

Any differences (misalignment) between the two distributions are thus amplified by the optimization process \eq{pxzopt} and force the optimization process \eq{pzxopt} to focus specifically on removing these differences. Since the misalignment between $\Px$ and $\g(\Pz)$ is measured after projecting the two distributions into the latent space, the maximization of this misalignment makes the encoder to compute features that distinguish the two distributions.

\section{Experiments}\label{s:exp}

\begin{figure*}
    \centering
    \begin{subfigure}[b]{0.24\linewidth}
    \centering
        \includegraphics[width=1\linewidth]{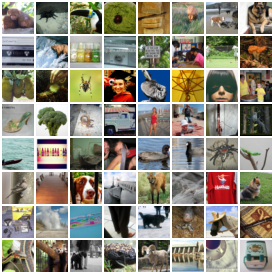}
        \caption{Real images}
    \end{subfigure}
    \begin{subfigure}[b]{0.24\linewidth}
    \centering
        \includegraphics[width=1\linewidth]{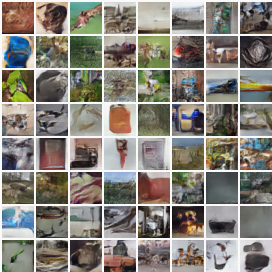}
        \caption{AGE samples}
    \end{subfigure}
    \begin{subfigure}[b]{0.24\linewidth}
    \centering
        \includegraphics[width=1\linewidth]{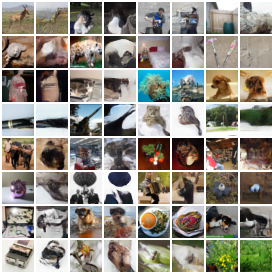}
        \caption{[Real, AGE reconstr.]}
    \end{subfigure}
    \begin{subfigure}[b]{0.24\linewidth}
    \centering
        \includegraphics[width=1\linewidth]{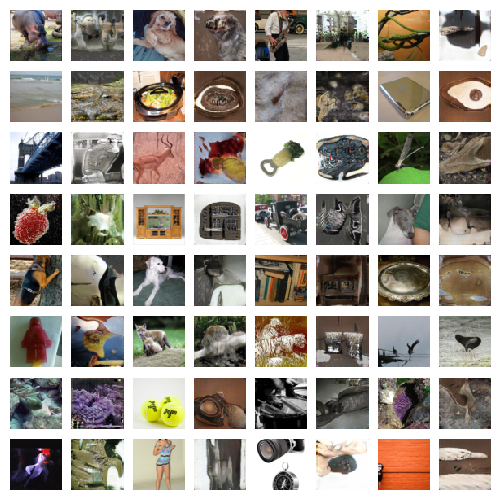}
        \caption{[Real, ALI reconstr.]}
    \end{subfigure}\\
    \begin{subfigure}[b]{0.24\linewidth}
    \centering
        \includegraphics[width=1\linewidth]{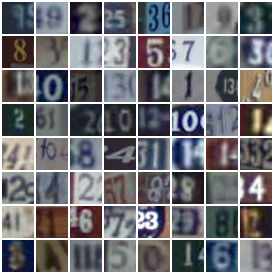}
    \end{subfigure}
    \begin{subfigure}[b]{0.24\linewidth}
    \centering
        \includegraphics[width=1\linewidth]{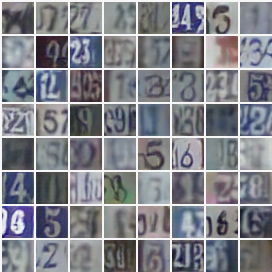}
    \end{subfigure}
    \begin{subfigure}[b]{0.24\linewidth}
    \centering
        \includegraphics[width=1\linewidth]{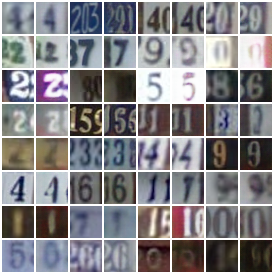}
    \end{subfigure}
    \begin{subfigure}[b]{0.24\linewidth}
    \centering
      \includegraphics[width=1\linewidth]{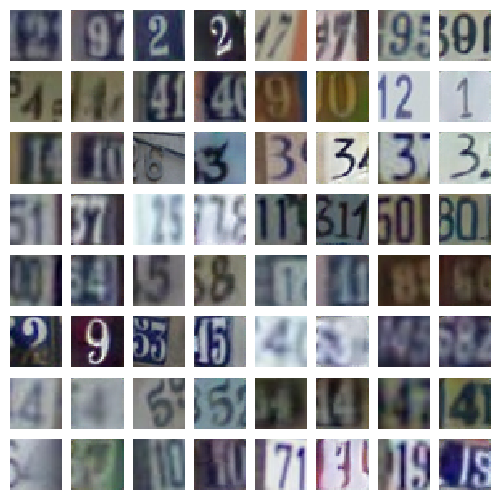}
    \end{subfigure}
 
    \caption{Samples (b) and reconstructions (c) for Tiny ImageNet dataset (top) and SVHN dataset (bottom). The results of ALI~\cite{Dumoulin16} on the same datasets are shown in (d). In (c,d)\, odd columns show real examples and even columns show their reconstructions. Qualitatively, our method seems to obtain more accurate reconstructions than ALI~\cite{Dumoulin16}, especially on the Tiny ImageNet dataset, while having samples of similar visual quality. }\label{fig:imagenet_svhn}
\end{figure*}

We have validated AGE networks in two settings. A more traditional setting involves \textit{unconditional} generation and reconstruction, where we consider a number of standard image datasets. We have also evaluated AGE networks in the \textit{conditional} setting. Here, we tackle the problem of image colorization, which is hard for GANs. In this setting, we condition both the generator and the encoder on the gray-scale image. Taken together, our experiments demonstrate the versatility of the AGE approach.

\subsection{Unconditionally-trained AGE networks}

{\bf Network architectures:} In our experiments, the generator and the encoder networks have a similar structure to the generator and the discriminator in DCGAN~\cite{Radford15}. To turn the discriminator into the encoder, we have modified it to output an $M$-dimensional vector and replaced the final sigmoid layer with the normalization layer that projects the points onto the sphere. 

\textbf{Divergence measure:} As we need to measure the divergence between the empirical distribution and the prior distribution in the latent space, we have used the following measure. Given a set of samples on the $M$-dimensional sphere, we fit the Gaussian Normal distribution with diagonal covariance matrix in the embedding $M$-dimensional space and we compute the KL-divergence of such Gaussian with the unit Gaussian as 
\begin{equation}
\KL(Q\|\mathcal{N}(0,I)) = - \frac{M}{2} + \frac{1}{M}\sum_{j=1}^M  \frac{s_j^2 + m_j^2}{2} - \log(s_j)\,
\end{equation}
where $m_j$ and $s_j$ are the means and the standard deviations of the fitted Gaussians along various dimensions. Since the uniform distribution on the sphere will entail the lowest possible divergence with the unit Gaussian in the embedding space among all distributions on the unit sphere, such divergence measure is valid for our analysis above. We have also tried to measure the same divergence non-parametrically using Kozachenko-Leonenko estimator~\cite{Kozachenko87}. In our initial experiments, both versions worked equally well, and we used a simpler parametric estimator in the presented experiments.

\textbf{Hyper-parameters:} We use ADAM~\cite{Kingma14} optimizer with the learning rate of $0.0002$. We perform two generator updates per one encoder update for all datasets. For each dataset we tried $\lambda \in \{500, 1000, 2000\}$ and picked the best one. We ended up using $\mu=10$ for all datasets. The dimensionality $M$ of the latent space was manually set according to the complexity of the dataset. We thus used $M=64$ for CelebA and SVHN datasets, and $M=128$ for the more complex datasets of Tiny ImageNet and CIFAR-10.

\textbf{Results:} We evaluate unconditional AGE networks on several standard datasets, while treating the system \cite{Dumoulin16} as the most natural reference for comparison (as the closest three-component counterpart to our two-component system). The results for \cite{Dumoulin16} are either reproduced with the author's code  or copied from \cite{Dumoulin16}.

In \fig{imagenet_svhn}, we present the results on the challenging Tiny ImageNet dataset~\cite{RussakovskyDSKS15} and the SVHN dataset~\cite{Netzer}. We show both samples $\g(\z)$ obtained for $\z \sim Z$ as well as the reconstructions $\g \left(\strut\e(\x)\right)$ alongside the real data samples $\x$. We also show the reconstructions obtained by \cite{Dumoulin16} for comparison. Inspection reveals that the fidelity of \cite{Dumoulin16} is considerably lower for Tiny ImageNet dataset.

\begin{figure*}
    \centering
    \includegraphics[width=1\linewidth]{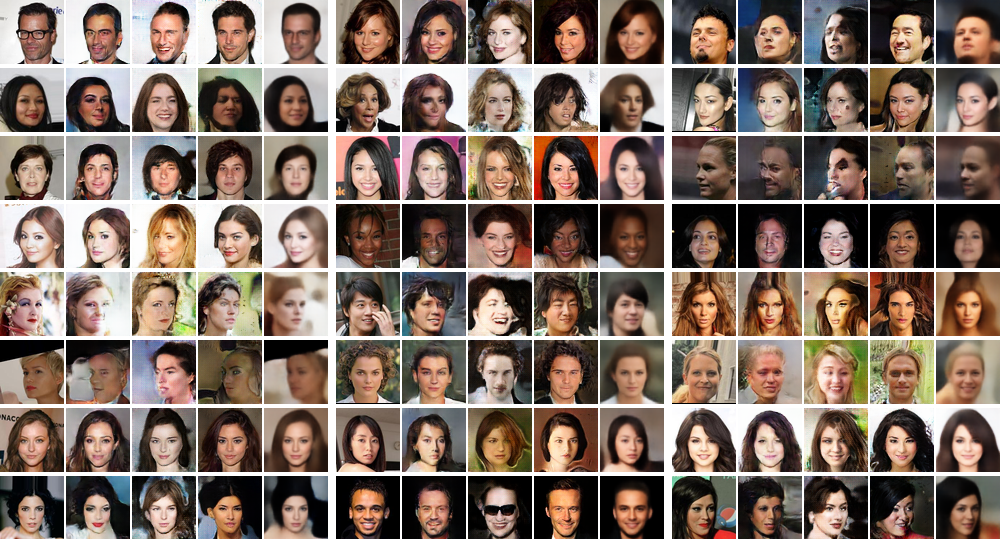}
\newcolumntype{A}{>{\centering\arraybackslash}p{0.0655\textwidth}}
\newcolumntype{B}{>{\centering\arraybackslash}p{0.0760\textwidth}}
\setlength{\tabcolsep}{0pt}

    \vspace*{-5mm}\begin{tabularx}{\textwidth}{AAAAABAAAABAAAA}
    \tiny{Orig.} &
    \tiny{AGE 10~ep.} &
    \tiny{ALI\, 10~ep.} &
    \tiny{ALI 100~ep.} &
    \tiny{VAE} &
    \tiny{Orig.} &
    \tiny{AGE 10~ep.} &
    \tiny{ALI\, 10~ep.} &
    \tiny{ALI 100~ep.} &
    \tiny{VAE} &
    \tiny{Orig.} &
    \tiny{AGE 10~ep.} &
    \tiny{ALI\, 10~ep.} &
    \tiny{ALI 100~ep.} &
    \tiny{VAE} 
    \end{tabularx}
    
    \caption{Reconstruction quality comparison of our method with ALI \cite{Dumoulin16} and VAE \cite{Kingma13}. The first column in each set shows examples from the test set of  CelebA dataset. In the other columns the reconstructions for different methods are presented: column two for ours method, three and four for ALI and five for VAE. We train our model for $10$ epochs and compare to ALI, trained for the same number of epochs (column three). Importantly one epoch for our method takes $3$ times less time than for ALI. For a fair comparison we also present the results of ALI, trained till convergence.}\label{fig:celeba_reconstructions} 
\end{figure*}

In \fig{celeba_reconstructions}, we further compare the reconstructions of CelebA~\cite{LiuLWT15} images obtained by the AGE network, ALI \cite{Dumoulin16}, and VAE~\cite{Kingma13}. Overall, the fidelity and the visual quality of AGE reconstructions are roughly comparable or better than ALI. Furthermore, ALI takes notoriously longer time to converge than our method, and the reconstructions of ALI after 10 epochs (which take six hours) of training look considerably worse than AGE reconstructions after 10 epochs (which take only two hours), thus attesting to the benefits of having a simpler two-component system.

Next we evaluate our method quantitatively. For the model trained on CIFAR-10 dataset we compute Inception score \cite{Salimans16}. The AGE score is $5.90 \pm 0.04$, which is higher than the ALI~\cite{Dumoulin16} score of $5.34 \pm 0.05$ (as reported in \cite{WardeFarley17}) and than the score of $4.36 \pm 0.04$ from \cite{Salimans16}. The state-of-the-art from \cite{WardeFarley17} is higher still ($7.72 \pm 0.13$). Qualitative results of AGE for CIFAR-10 and other datasets are shown in supplementary material.

We also computed log likelihood for AGE and ALI on the MNIST dataset using the method of \cite{Wu16} with latent space of size $10$ using authours source code. ALI’s score is $721$ while AGE’s score is $746$. The AGE model is also superior than both VAE and GAN, which scores are $705.375$ and $346.679$ respectively as evaluated by \cite{Wu16}.

Finally, similarly to \cite{Dumoulin16, Donahue16, Radford15} we investigated whether the learned features are useful for discriminative tasks. We reproduced the evaluation pipeline from \cite{Dumoulin16} for SVHN dataset and obtained $23.7\%$ error rate in the unsupervised feature learning protocol with our model, while their result is $19.14\%$. At the moment, it is unclear to us why AGE networks underperform ALI at this task.

\subsection{Conditional AGE network experiments.}

Recently, several GAN-based systems have achieved very impressive results in the conditional setting, where the latent space is augmented or replaced with a second data space corresponding to different modality~\cite{Isola16,Zhu17}. Arguably, it is in the conditional setting where the bi-directionality lacking in conventional GANs is most needed. In fact, by allowing to switch back-and-forth between the data space and the latent space, bi-directionality allows powerful neural image editing interfaces~\cite{Zhu16,Brock16}.

Here, we demonstrate that AGE networks perform well in the conditional setting. To show that, we have picked the image colorization problem, which is known to be hard for GANs. To the best of our knowledge, while the idea of applying GANs to the colorization task seems very natural, the only successful GAN-based colorization results were presented in \cite{Isola16}, and we compare to the authors' implementation of their pix2pix system. We are also aware of several unsuccessful efforts to use GANs for colorization.

To use AGE for colorization, we work with images in the \textit{Lab} color space, and we treat the \textit{ab} color channels of an image as a data sample $\x$. We then use the lightness channel $L$ of the image as an input to both the encoder $\eofx{\x|L}$ and the generator $\gofz{\z|L}$, effectively conditioning the encoder and the generator on it. Thus, different latent variables $\z$ will result in different colorizations $\x$ for the same grayscale image $L$. The latent space in these experiments is taken to be three-dimensional.

The particular architecture of the generator takes the input image $L$, augments it with $\z$ variables expanded to constant maps of the same spatial dimensions as $L$, and then applies the ResNet type architecture~\cite{He16, Johnson16} that computes $\x$ (i.e.\ the \textit{ab}-channels). The encoder architecture is a convolutional network that maps the concatenation of $L$ and $\x$ (essentially, an image in the Lab-space) to the latent space. The divergence measure is the same as in the unconditional AGE experiments and is computed ``unconditionally'' (i.e.\ each minibatch passed through the encoder combines multiple images with different $L$).

We perform the colorization experiments on Stanford Cars dataset~\cite{Krause13} with 16,000 training images of 196 car models, since cars have inherently ambiguous colors and hence their colorization is particularly prone to the regression-to-mean effect. The images were downsampled to $64{\times}64$. 

We present colorization results in \fig{cars}. Crucially, AGE generator is often able to produce plausible and diverse colorizations for different latent vector inputs. 
As we wanted to enable pix2pix GAN-based system of \cite{Isola16} to produce diverse colorizations, we augmented the input to their generator architecture with three constant-valued maps (same as in our method). We however found that their system effectively learns to ignore such input augmentation and the diversity of colorizations was very low (\fig{cars}a).

To demonstrate the meaningfulness of the latent space learned by the conditional AGE training, we also demonstrate the color transfer examples, where the latent vector $\z_1 = \eofx{\x_1|L_1}$ obtained by encoding the image $[x_1,L_1]$ is then used to colorize the grayscale image $L_2$, i.e.\ $\x_2 = \gofz{\z_1|L_2}$ (\fig{cars}b).

\begin{figure}
    \centering
    \begin{subfigure}[b]{0.79\linewidth}
    \centering
        \includegraphics[width=1\linewidth]{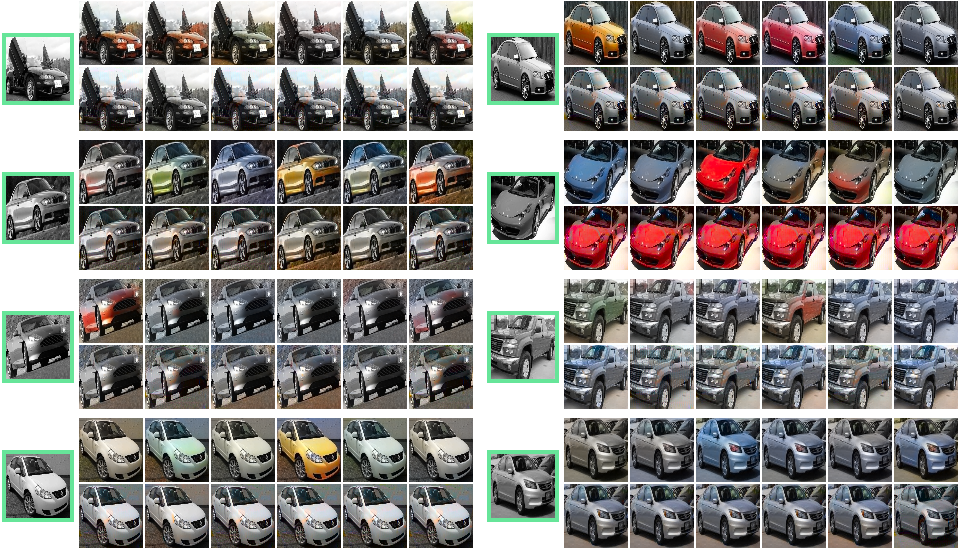}
        \caption{Colorizations -- AGE network (top rows) vs. pix2pix \cite{Isola16} (bottom rows)}
    \end{subfigure}
    \begin{subfigure}[b]{0.20\linewidth}
    \centering
        \includegraphics[width=0.99\linewidth]{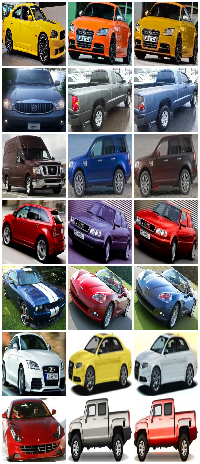}
        \caption{Color transfer}
    \end{subfigure}  
    \caption{(a) Each pane shows colorizations of the input grayscale image (emphasized) using conditional AGE networks (top rows) and pix2pix~\cite{Isola16} with added noise maps (bottom rows). AGE networks produce diverse colorizations, which are hard to obtain using pix2pix. (b) In each row we show the result of color transfer using the conditional AGE network. The color scheme from the first image is transferred onto the second image.}
    \label{fig:cars} 
\end{figure}

\section{Conclusion}\label{s:conc}

We have introduced a new approach for simultaneous learning of generation and inference networks. We have demonstrated how to set up such learning as an adversarial game between generation and inference, which has a different type of objective from traditional GAN approaches. In particular the objective of the game considers divergences between distributions rather than discrimination at the level of individual samples. As a consequence, our approach does not require training a discriminator network and enjoys relatively quick convergence.



We demonstrate that on a range of standard datasets, the generators obtained by our approach provides high-quality samples, and that the reconstructions of real data samples passed subsequently through the encoder and the generator are of better fidelity than in \cite{Dumoulin16}. We have also shown that our approach is able to generate plausible and diverse colorizations, which is not possible with the GAN-based system~\cite{Isola16}.

Our approach leaves a lot of room for further experiments. In particular, a more complex latent space distribution can be chosen as in \cite{Makhzani15}, and other divergence measures between distributions can be easily tried.


\bibliographystyle{plain}
\bibliography{refs}

\vfill
\pagebreak

\section{Appendix}
\appendix
In this supplementary material, we provide proofs  for the theorems of the main text (restating these theorems for convenience of reading). We also show additional qualitative results on several datasets.

\section{Proofs}

\newcommand{\Pbb}{{\mathbb{P}}}
\newcommand{\Q}{{\mathbb{Q}}}
\newcommand{\B}{{\mathcal{B}}}
\newcommand{\R}{{\mathbb{R}}}
\newcommand{\F}{{\mathcal{F}}}

\newtheorem{innercustomthm}{Theorem}
\newenvironment{customthm}[1]
  {\renewcommand\theinnercustomthm{#1}\innercustomthm}
  {\endinnercustomthm}
  
\newtheorem{theorem1}{Theorem}[section]
\newtheorem{lemma}[theorem1]{Lemma}

Let $X$ and $Z$ be distributions defined in the data and the latent spaces $\sX$, $\sZ$ correspondingly. We assume $X$ and $Z$ are such, that there exists an invertible almost everywhere function $e$ which transforms the latent distribution into the data one $g(Z)=X$. This assumption is weak, since for every atomless (i.e. no single point carries a positive mass) distributions $X$, $Z$ such invertible function exists. For a detailed discussion on this topic please refer to \cite{marzouk2016introduction,villani2008optimal}. Since $Z$ is up to our choice simply setting it to Gaussian distribution (for $\sZ = \R^M$) or uniform on sphere for ($\sZ = \sph)$ is good enough.


\begin{lemma}\label{lemma:all_e}
Let $X$ and $Y$ to be two distributions defined in the same space. 
The distributions are equal $X = Y$ if and only if $e(X) = e(Y)$ holds for for any measurable function $e: \sX \rightarrow \sZ$. 
\end{lemma}
 \begin{proof}
It is obvious, that if $X=Y$ then $e(X)=e(Y)$ for any measurable function $e$. 


Now let $e(X)=e(Y)$ for any measurable $e$. To show that $X=Y$ we will assume converse: $X \neq Y$. Then there exists a set $B \in \mathcal{F}_X$, such that $0 < \Pbb_X(B) \neq \Pbb_Y(B)$ and a function $e$, such that corresponding set $C=e(B)$ has $B$ as its preimage $B = e^{-1}(C)$. Then we have $\Pbb_{X}(B) = \Pbb_{e(X)}(C) = \Pbb_{e(Y)}(C) = \Pbb_Y(B)$, which contradicts with the previous assumption. 
\end{proof}

\begin{lemma}\label{lemma:same_value}
Let $(\g',\e')$ and $(\g^*,\e^*)$ to be two different Nash equilibria in a game $\min_\g \max_\e V(\g,\e)$. Then $V(\g,\e) = V(\g', \e')$.  
\end{lemma}
 
\begin{proof} See chapter 2 of \cite{owen1982game}.
\end{proof}

\begin{customthm}{1}
For a game
\begin{equation}\label{eq:appendix_game1}
\max_\e \min_\g V_1(\g,\e) = \Delta(\,\Pegz \| \Pex\,)
\end{equation}
$(\g^*,\e^*)$ is a saddle point of \eq{appendix_game1} if and only if $\g^*$ is such that $\g^*(Z)=X$.  
\end{customthm}

\begin{proof}
First note that $V_1(\g,\e) \geq 0$. Consider $\g$ such that $\Pgz=\Px$, then for any $\e$: $V_1(\g,\e)= 0$. We conclude that $(\g,\e)$ is a saddle point since $V_1(\g,\e)=0$ is a maximum over $\e$ and minimum over $\g$. 

Using lemma \ref{lemma:same_value} for saddle point $(\g^*, \e^*)$:  $V_1(\g^*, \e^*)= 0 = \max_e V_1(\g^*, \e)$, which is only possible if for all $\e$: $V_1(\g^*, \e)=0$ from which immediately follows $\g(Z)=X$ by lemma \ref{lemma:all_e}.
\end{proof}

\begin{lemma}\label{lemma:inverse}
Let function $e: \sX \rightarrow \sZ$ be $X$-almost everywhere invertible, i.e. $\exists e^{-1}: \Pbb_X(\{\x \neq e^{-1}(e(\x))\}) = 0$. Then if for a mapping $g: \sZ \rightarrow \sX$ holds $e(g(Z)) = e(X)$, then $g(Z) = X$.
\end{lemma}
\begin{proof}
From definition of $X$-almost everywhere invertibility follows $\Pbb_X(A) = \Pbb_X(e^{-1}(e(A)))$ for any set $A \in \mathcal{F}_X$. Then:
\begin{gather*}
\Pbb_X(A) =  \Pbb_X(e^{-1}(e(A))) =  \Pbb_{e(X)}(e(A)) = \\ =  \Pbb_{e(g(Z))}(e(A)) = \Pbb_{g(Z)}(A).
\end{gather*}
Comparing the expressions on the sides we conclude $g(Z) = X$. 


\end{proof}

\begin{customthm}{2}
Let $Y$ to be any fixed distribution in the latent space. Consider a game
\begin{equation}\label{eq:appendix_game2}
    \max_e \min_g V_2(g,e) = \Delta(e(g(Z)) \| Y) - \Delta(e(X) \| Y)\, .
\end{equation}

If the pair $(g^*,e^*)$ is a Nash equilibrium of game \eq{appendix_game2} then $g^*(Z) \sim X$. Conversely, if the fake and real distributions are aligned $g^*(Z) \sim X$ then $(g^*,e^*)$ is a saddle point for \textit{some} $e^*$. 
\end{customthm}

\begin{proof} 
\
\begin{itemize}
    \item 
    
    As for a generator which aligns distributions $g(Z) = X$: $V_2(g,e) = 0$ for any $e$ we conclude by \ref{lemma:same_value} that the optimal game value is $V_2(g^*,e^*) = 0$. For an optimal pair $(g^*, e^*)$ and arbitrary $e'$ from the definition of equilibrium: 
    \begin{equation}\label{eq:ineq}
    \begin{array}{c}  
    0 = \Delta(e^*(g^*(Z)) \| Y) - \Delta(e^*(X) \| Y ) \geq \\ \geq \Delta(e'(g^*(Z))\| Y) - \Delta(e'(X)\| Y)\, .
    \end{array}
    \end{equation}
    For invertible almost everywhere encoder $e'$ such that $\Delta (e'(X) \| Y) = 0$ the first term is zero $\Delta (e'(g^*(Z)) \| Y) = 0$ since inequality \eq{ineq} and then $e'(g^*(Z)) = e'(X) = Y$. Using result of the lemma \ref{lemma:inverse} we conclude, that $g^*(Z) = X$.   
    

    \item If $g^*(Z) = X$ then $\forall e: e(g^*(Z)) = e(X)$ and $V_2(g^*, e^*) = V_2(g^*, e) = 0 = \max_{e'} V_2(g^*, e')$. 
    
    The corresponding optimal encoder $e^*$ is such that ${g^* \in \arg\min_g \Delta (e^*(g(Z))\| Y)}$.
\end{itemize}
\end{proof}

Note that not for every optimal encoder $e^*$ the distributions $e^*(X)$ and $e^*(g^*(Z))$ are aligned with $Y$. For example if $e^*$ collapses $\sX$ into two points then for any distribution $X$: $e^*(X) = e^*(g^*(Z)) = Bernoulli(p)$. For the optimal generator $g^*$ the parameter $p$ is such, that for all other generators $g'$ such that $e^*(g'(Z)) \sim Bernoulli(p')$: $\Delta(e^*(g^*(Z))\| Y) \leq \Delta(e^*(g'(Z))\| Y)$.


\begin{customthm}{3}
Let the two distributions $W$ and $Q$ be aligned by the mapping $f$ (i.e.\ $f(W) = Q$) and let $\E_{\w \sim W} \| \w - h \left(\strut f(\w)\right) \|^2 = 0$. Then, for $\w \sim W$ and $\q \sim Q$, we have $\w = h(f(\w))$ and $\q = f(h(\q))$ almost certainly, i.e.\ the mappings $f$ and $h$ invert each other almost everywhere on the supports of $W$ and $Q$. More, $Q$ is aligned with $W$ by $h$: $h(Q) = W$. 
\end{customthm}
\begin{proof}
Since $\E_{\w \sim W} \| \w - h \left(\strut f(\w)\right) \|^2 = 0$,  we have
$\w = h(f(\w))$ almost certainly for $\w \sim W$ . Using this and the fact that $f(\w) \sim Q$ for $\w \sim W$ we derive:   
\begin{gather*}
\E_{\q \sim Q} \| \q - f \left(\strut h(\q)\right) \|^2 = \E_{\w \sim W} \| f(\w) - f \left(\strut h(f(\w))\right) \|^2 = \\ = \E_{\w \sim W} \| f(\w) - f (\w) \|^2 = 0\,.
\end{gather*}
Thus $\q = f(h(\q))$ almost certainly for $\q \sim Q$.

To show alignment $h(Q) = W$ first recall the definition of alignment. Distributions are aligned $f(W) = Q$ iff $\forall \bar Q \in \mathcal{F}_Q$: $\Pbb_Q(\bar Q) = \Pbb_W(f^{-1}(\bar Q))$. Then $\forall \bar W \in \mathcal{F}_W$ we have
\begin{gather*}
\Pbb_W(\bar W) = \Pbb_W(h(f(\bar W))) =  \Pbb_W(f^{-1}(f(\bar W)))  = \\ = \Pbb_Q(f(\bar W)) =  \Pbb_Q(h^{-1}(\bar W))\,.
\end{gather*}
Comparing the expressions on the sides we conclude $h(Q) = W$.
\end{proof}

\section{Additional qualitative results.}

In the figures, we present additional qualitative results and comparisons for various image datasets. See figure captions for explanations.
\FloatBarrier

\begin{figure*}
    \centering
    \begin{subfigure}[b]{0.4\linewidth}
        \includegraphics[width=\linewidth]{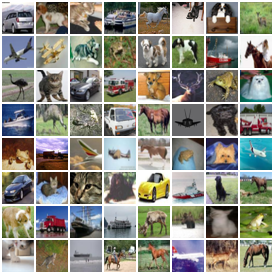}
        \caption{Real}
    \end{subfigure}\qquad
    \begin{subfigure}[b]{0.4\linewidth}
        \includegraphics[width=\linewidth]{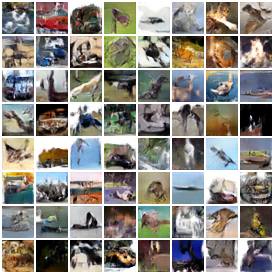}
        \caption{DCGAN \cite{Radford15}}
    \end{subfigure}\vspace{2mm}\\
    \begin{subfigure}[b]{0.4\linewidth}
        \includegraphics[width=\linewidth]{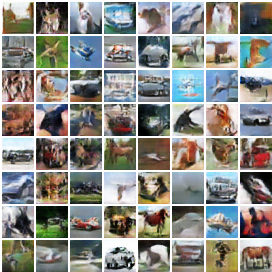}
        \caption{AGE (distribution alignment)}
    \end{subfigure}\qquad
    \begin{subfigure}[b]{0.4\linewidth}
        \includegraphics[width=\linewidth]{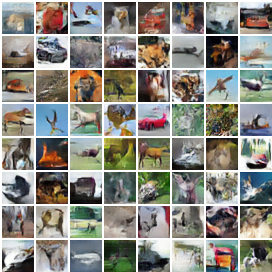}
        \caption{AGE (full)}
    \end{subfigure}
    \caption{We compare CIFAR-10 samples from DCGAN~\cite{Radford15}~(b) to the samples generated using our ablated model trained without reconstruction terms in (c) using distribution alignment only. The model, trained with the reconstruction terms is still able to produce diverse samples~(d), but also allows inference (\fig{cifar10_inference}). }\label{fig:cifar10}
\end{figure*}

\begin{figure}
    \centering
    \begin{subfigure}[b]{0.4\linewidth}
        \includegraphics[width=\linewidth]{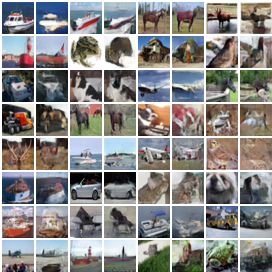}
        \caption{[Real, AGE network]}
    \end{subfigure}\qquad
    \begin{subfigure}[b]{0.4\linewidth}
        \includegraphics[width=\linewidth]{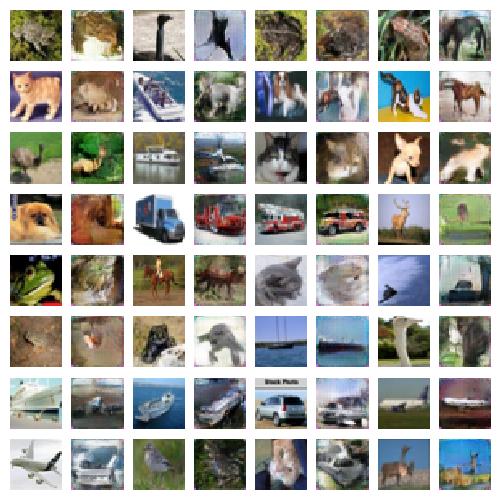}
        \caption{[Real, ALI \cite{Dumoulin16}]}
    \end{subfigure}
    \caption{Comparison in reconstruction quality to ALI~\cite{Dumoulin16} for the CIFAR-10 dataset. For both figures real examples are shown in odd columns and their reconstructions are shown in even columns. The real examples come from test set and were never observed by the model during training. }\label{fig:cifar10_inference}
\end{figure}


\begin{figure}
    \centering
    \begin{subfigure}[b]{0.48\linewidth}
    \centering
        \includegraphics[width=1\linewidth]{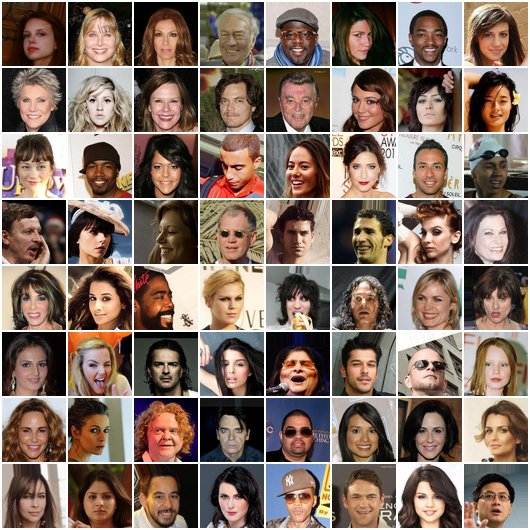}
        \caption{Real images}
    \end{subfigure} \quad
    \begin{subfigure}[b]{0.48\linewidth}
    \centering
        \includegraphics[width=1\linewidth]{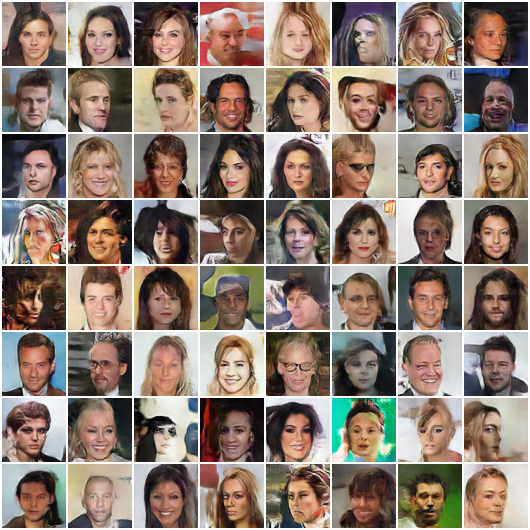}
        \caption{AGE samples}
    \end{subfigure}

    \caption{Real examples (a) and samples (b) from our model trained on CelebA dataset.}\label{fig:celeba_samples} 
\end{figure}

\begin{figure}
    \centering
    \includegraphics[width=0.7\linewidth]{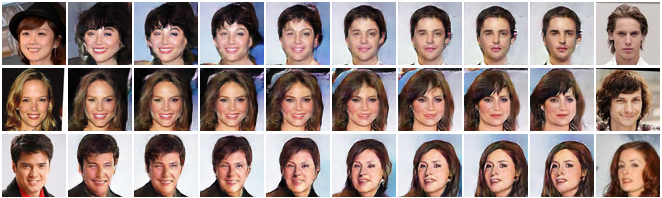}
  \caption{Latent space interpolation between two images from CelebA dataset. The original images are presented on the two sides.}\label{fig:latent}
\end{figure}

\begin{figure}
    \centering
        \includegraphics[width=0.6\linewidth]{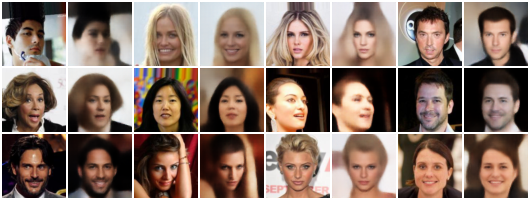}
    \caption{In all experiments except this one, we do not use data space reconstruction loss in the objective of the generator. These figure demonstrates the degradation occurring when this term is added. Odd columns correspond to real images and even to reconstructions.}\label{fig:l2loss}
\end{figure}

\begin{figure}
    \centering
        \includegraphics[width=0.7\linewidth]{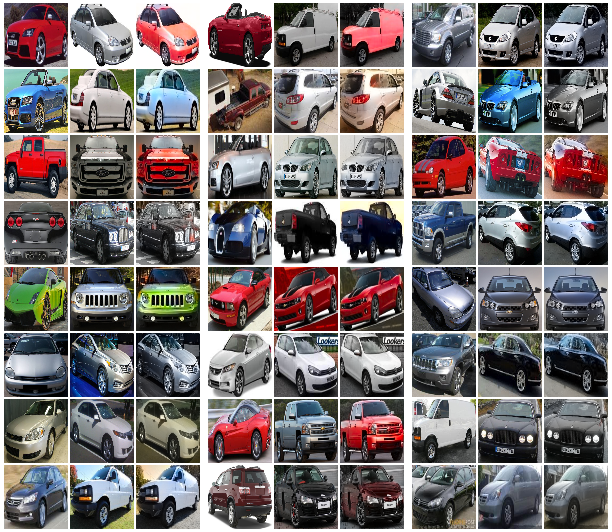}
    \caption{More color transfer results: in each triplet the color scheme is transferred from the first image onto the second image using the conditional AGE model.}
\end{figure}
\end{document}